\documentclass[11pt,a4paper]{article}

\usepackage[utf8]{inputenc}
\usepackage[T1]{fontenc}
\usepackage{amsmath,amssymb,amsfonts,amsthm}
\usepackage[numbers,sort&compress]{natbib}
\usepackage{graphicx}
\usepackage{geometry}
\usepackage{booktabs}
\usepackage{algorithm}
\usepackage{algpseudocode}
\usepackage{xcolor}
\usepackage{url}
\usepackage[colorlinks=true,citecolor=blue,urlcolor=blue]{hyperref}
\usepackage{multirow}

\newcommand{\chisao}{\textsc{ChiSao}}

\newcommand{\RR}{\mathbb{R}}
\DeclareMathOperator*{\argmax}{arg\,max}

\newtheorem{proposition}{Proposition}
\newtheorem{remark}{Remark}

\begin{document}

\title{\chisao{}: A GPU-Native Parallel Optimizer for Multimodal\\
Black-Box Functions via Convergence-Anticonvergence Oscillation}

\author{Ira Wolfson\\
  Department of Electronics and Electrical Engineering\\
  Braude College of Engineering, Karmiel, Israel\\
  \texttt{wolfsoni@braude.ac.il}}

\maketitle

\begin{abstract}
Finding all modes of a multimodal black-box function is a fundamental challenge in optimization, Bayesian inference, and scientific computing. Existing approaches---basin-hopping, CMA-ES, multistart gradient descent---operate sequentially and cannot exploit the massive parallelism of modern GPU hardware. We introduce \chisao{} (\textbf{C}onvergence-\textbf{H}alt-\textbf{I}nvert-\textbf{S}tick-\textbf{A}nd-\textbf{O}scillate), a GPU-native population optimizer that runs an entire sample batch simultaneously and exploits a deliberate convergence-anticonvergence oscillation cycle to escape local traps while freezing confirmed modes. The structural move is asymmetric: samples that reach true peaks are frozen (``stuck'') and preserved, while the rest keep exploring via momentum-based anti-convergence and stochastically smoothed gradients. Adaptive reseeding via two complementary strategies (Repulse Monkey and Golden Rooster) maintains population diversity throughout. On all 42 functions of the Simon Fraser University optimization benchmark suite across dimensions $d \in \{2, 4, 8, 16, 32, 64\}$, \chisao{} achieves \textbf{100\%} mode recovery where all CPU baselines collapse at $d \geq 8$ on the hardest multimodal functions, at up to \textbf{$34\times$} speedup over basin-hopping on functions where all methods succeed (Michalewicz $d=64$) and up to \textbf{$39\times$} on unimodal functions (Rotated Hyper-Ellipsoid $d=64$, pure GPU dividend). All benchmarks evaluate the objective by value alone---gradients come from finite differences---so the reported speedups are a derivative-free worst case. Under substantial likelihood noise ($\sigma_{\mathrm{noise}}$ up to 1.0), mode detection remains 100\% reliable. The algorithm is available as a standalone open-source Python package on PyPI.
\end{abstract}

\medskip\noindent\textbf{Keywords:} multimodal optimization, GPU computing, black-box optimization, population methods, mode finding, parallel optimization

%% ============================================================
\section{Introduction}
\label{sec:intro}
%% ============================================================

Multimodal optimization---finding all significant maxima of a function $f: \RR^d \to \RR$ accessible only through black-box evaluations---arises across scientific computing, probabilistic inference, and machine learning. The challenge is not merely finding a global optimum but \emph{cataloguing all modes}, since downstream tasks (Bayesian model averaging, multistart refinement, mixture fitting) require knowledge of the full mode structure.

Classical approaches fall into two families. \emph{Sequential} methods such as basin-hopping \citep{wales1997global} and simulated annealing \citep{kirkpatrick1983optimization} perturb a single solution and accept or reject moves stochastically; they explore one trajectory at a time and are inherently sequential. \emph{Population} methods such as CMA-ES \citep{hansen2001completely} and particle swarm optimization \citep{kennedy1995particle} maintain a set of candidate solutions but typically converge the entire population toward a single basin, losing mode diversity. Multistart gradient descent runs multiple independent optimization trajectories but these are embarrassingly parallel only in the trivial sense---each trajectory is independent and discovers at most one mode.

None of these is built for the architecture of modern GPUs: thousands of arithmetic units running the same instruction on different data simultaneously. On a GPU with $P > 10^4$ parallel cores, a batch of $N \leq P$ gradient evaluations completes in the same wall-clock time as a single evaluation. Sequential methods waste this capacity entirely; population methods use it only within a single generation step.

\chisao{}\footnote{Wing Chun \emph{Ch\'{i} S\v{a}o} (``Sticky Hands''): in the martial art, sensitivity drills train simultaneous contact maintenance and redirection, precisely the explore-and-freeze asymmetry of the algorithm.} is a GPU-native population optimizer with three structural moves no prior method combines.

The first is a freeze-and-explore asymmetry. A sample that reaches a true peak, passing both a gradient-norm and a likelihood-quality test, is frozen from the exploration phases but continues to participate in gradient ascent for peak refinement. The rest of the population keeps moving. Standard population methods either advance all particles or none; \chisao{} advances only those that still have work to do.

The second is a deliberate anti-convergence phase. After each batched L-BFGS pass, unfrozen samples take momentum-based gradient \emph{descent} steps. This is neither noise injection (simulated annealing) nor random perturbation (basin-hopping); it is directed motion, the momentum term carrying samples across valleys into new basins.

The third is stochastic smoothing, applied as a Hands Like Clouds phase in which unfrozen samples ascend on a Gaussian-smoothed estimate of $\nabla f$. Smoothing at scale $\sigma$ erases sub-$\sigma$ texture and exposes the global basin geometry. HLC and anti-convergence are complementary rather than redundant: HLC steers samples toward promising regions on the smoothed landscape; anti-convergence then disperses them on the raw landscape.

Two reseeding strategies---Repulse Monkey when many samples remain unfrozen, Golden Rooster when most are frozen---close the cycle. The algorithm needs only batched function and gradient evaluations, maps directly onto GPU execution, and works as a drop-in exploration module for any optimizer that supplies an initial population.

\paragraph{Contributions.} The contribution is fourfold: the oscillation cycle as a strategy for multimodal optimization with a formalized freeze-and-explore asymmetry (Section~\ref{sec:algorithm}); convergence guarantees for log-concave targets and a coverage analysis for the multimodal case (Section~\ref{sec:theory}); benchmarks against Differential Evolution, Basin-Hopping, and CMA-ES on all 42 SFU functions across $d \in \{2,4,8,16,32,64\}$ (Section~\ref{sec:experiments}); and noise-robustness results showing 100\% mode detection up to $\sigma_{\mathrm{noise}} = 1.0$, the signal scale itself (Section~\ref{sec:noise}).

%% ============================================================
\section{Related Work}
\label{sec:related}
%% ============================================================

\paragraph{Basin-hopping.} Wales and Doye \citep{wales1997global} introduced basin-hopping as alternating random perturbations with local minimization, accepting moves via a Metropolis criterion. It is effective in low dimensions and has been widely used in computational chemistry \citep{wales2003energy}. Its fundamental limitation is sequential execution: one perturbation, one local minimization, one accept/reject decision per step. GPU parallelization is limited to the local minimization sub-problem, not the search strategy.

\paragraph{Simulated annealing.} Kirkpatrick et al.\ \citep{kirkpatrick1983optimization} introduced temperature-based stochastic acceptance, which provides theoretical convergence guarantees under slow cooling \citep{hajek1988cooling} but requires an exponentially slow schedule to avoid premature convergence in practice. Parallel implementations \citep{ram1996parallel} distribute independent chains but do not share information, losing the benefit of population diversity.

\paragraph{Evolutionary and population methods.} CMA-ES \citep{hansen2001completely} adapts a full covariance matrix to the population geometry, achieving state-of-the-art performance on unimodal and mildly multimodal problems. Niching extensions \citep{preuss2015multimodal} attempt to maintain diversity but add significant complexity. Particle swarm optimization \citep{kennedy1995particle} uses social attraction toward the population best, which concentrates the population at a single basin in multimodal settings unless modified with repulsion terms \citep{brits2002niching}. Differential evolution \citep{storn1997differential} generates candidates by combining population members but shares CMA-ES's tendency toward mode collapse.

\paragraph{Multistart methods.} Running multiple independent gradient-based optimizers from random initial points is the most common approach in practice \citep{marti2003multi}. It parallelizes trivially but has two weaknesses: restarts that converge to the same mode waste computation, and there is no mechanism to ensure the population maintains diversity after convergence begins. \chisao{}'s deduplication and reseeding directly address both.

\paragraph{Bayesian optimization.} Gaussian-process surrogate methods \citep{shahriari2016taking} target the same black-box setting from a serial angle: each acquisition step refits the surrogate against all prior evaluations and then optimises the acquisition function. The design objective is locating the single best point at minimum evaluation budget, not cataloguing modes, and the per-step surrogate cost grows poorly with population size.

\paragraph{GPU-parallel optimization.} GPU acceleration of optimization has focused primarily on stochastic gradient descent for deep learning \citep{goodfellow2016deep}, where the objective is unimodal and the bottleneck is data throughput rather than exploration. For black-box multimodal optimization, GPU acceleration is largely unexplored, with \citet{salimans2017evolution} the major exception---evolution strategies massively parallelized for reinforcement learning, but with a single-objective design that does not catalogue modes. \citet{albert2023jaxns} parallelized nested sampling on GPU (JAXNS) but the sequential compression structure of nested sampling limits GPU utilization to within-iteration batching. \chisao{} appears to be the first black-box multimodal optimizer with full GPU utilization across the search strategy itself, not only inside individual iterations.

\paragraph{Smoothing and continuation methods.} Gaussian smoothing of objective functions as a heuristic for escaping local optima has been studied under the name ``diffusion'' or ``graduated non-convexity'' \citep{blake1987visual, mobahi2015link}. \citet{nesterov2017random} analyzed Gaussian smoothing for convex optimization. The Hands Like Clouds phase of \chisao{} applies smoothing as one phase within a structured oscillation cycle, rather than as a global preprocessing step or annealing schedule.

\paragraph{Position relative to prior work.} \chisao{} differs from every method above on at least one structural axis: directional anti-convergence rather than random perturbation, stochastic smoothing within a cycle rather than a global annealing schedule, freeze-and-explore asymmetry rather than uniform population movement, and full-batch GPU execution rather than within-step parallelism. Basin-hopping with restarts \citep{wales1997global} is the closest neighbour; it differs on all four.

%% ============================================================
\section{Algorithm}
\label{sec:algorithm}
%% ============================================================

\subsection{Problem Statement}

Let $f: \Theta \to \RR$ be a black-box function on a bounded domain $\Theta \subset \RR^d$, accessible through evaluations $f(\theta)$; the gradients $\nabla f(\theta)$ the algorithm uses are computed by finite differences, or supplied analytically when available. We seek the set of all \emph{significant modes}:
\begin{equation}
    \mathcal{M}^* = \left\{ \theta^* \in \Theta : \nabla f(\theta^*) = 0,\; \nabla^2 f(\theta^*) \prec 0,\; f(\theta^*) \geq f_{\max} + \log \delta \right\}
\end{equation}
where $f_{\max} = \max_\theta f(\theta)$ and $\delta \in (0,1)$ is a quality threshold (default $\delta = 0.1$, i.e., modes within one decade of the global maximum). The goal is not merely to find the global maximum but to identify all members of $\mathcal{M}^*$.

\chisao{} takes as input a population $\{x_i\}_{i=1}^N \subset \Theta$ of initial candidate points (arbitrary; provided by the caller), and returns the estimated set $\hat{\mathcal{M}}^* \subseteq \{x_i\}$.

\subsection{The Oscillation Cycle}

\chisao{} runs for $n_{\mathrm{osc}}$ oscillation cycles (default 3). Each cycle consists of six phases executed in fixed order. The ordering is deliberate and is justified in Section~\ref{sec:rationale}.

\subsubsection{Phase 1: Convergence}

All stuck masks are released (every sample is temporarily unfrozen) and the entire population undergoes $n_{\mathrm{conv}}$ steps of batched L-BFGS \citep{liu1989lbfgs,nocedalwright2006} toward local maxima of $f$. All $N$ samples optimize simultaneously in a single GPU batch:
\begin{equation}
    x_i \gets \mathrm{L\text{-}BFGS}(x_i, n_{\mathrm{conv}}) \quad \forall i = 1, \ldots, N \quad \text{(GPU-parallel)}
\end{equation}
The L-BFGS memory parameter is $m = 10$ by default. Iteration count adapts to dimension: $n_{\mathrm{conv}} = \max(10, 3\log_2 d)$.

\subsubsection{Phase 2: Stick Detection}

After convergence, samples that have reached true local maxima are marked as \emph{stuck}:
\begin{equation}
    \mathrm{stuck}_i \gets \left( \|\nabla f(x_i)\|_\infty < \epsilon_{\mathrm{grad}} \right) \wedge \left( f(x_i) \geq f_{\max} + \log \delta \right)
\end{equation}
The gradient threshold $\epsilon_{\mathrm{grad}}$ (default $10^{-6}$) ensures the sample has reached a stationary point. The likelihood threshold $\log \delta$ (default $\log 0.1$) prevents sticking at low-quality local maxima that fall outside $\mathcal{M}^*$. Stuck samples are excluded from exploration phases but continue to participate in gradient ascent for peak refinement.

\subsubsection{Phase 3: Deduplication}

Stuck samples are deduplicated using the $L_\infty$ metric (see Appendix~\ref{app:linf}):
\begin{equation}
    \text{remove } x_j \text{ if } \exists\, x_i \neq x_j : \|x_i - x_j\|_\infty < \epsilon_{\mathrm{dup}},\; f(x_i) \geq f(x_j)
\end{equation}
The sample with the higher function value survives; its inverse Hessian estimate is preserved for width estimation. The number of removed duplicates $K_{\mathrm{lost}}$ is recorded.

\subsubsection{Phase 4: Reseeding (Repulse Monkey / Golden Rooster)}

If $K_{\mathrm{lost}} > 0$ and this is not the last oscillation, the population is maintained at size $N$ by reseeding. Two strategies are used depending on exhaustion state.

\paragraph{Repulse Monkey.} When $\geq 5$ unfrozen samples remain, new samples are generated by shooting rays from unfrozen samples in random directions, sampling uniformly along those rays within $\Theta$. This disperses new candidates away from known peaks.

\paragraph{Golden Rooster.} When $< 5$ unfrozen samples remain (near exhaustion), new samples are generated from confirmed peaks using orthonormal ray directions obtained via QR decomposition of a random matrix. This exploits the GPU's parallel capacity to systematically probe orthogonal directions from known peaks. Golden Rooster adaptively disables itself if successive generations discover no new peaks, preventing wasted computation.

All reseeded samples are marked as unstuck to enable exploration.

\subsubsection{Phase 5: Hands Like Clouds}

If this is not the last oscillation, unfrozen samples take $n_{\mathrm{cloud}}$ gradient \emph{ascent} steps using a stochastically smoothed gradient:
\begin{equation}
    \nabla f_\sigma(\theta) \approx \frac{1}{K} \sum_{k=1}^{K} \nabla f(\theta + \sigma z_k), \quad z_k \sim \mathcal{N}(0, I_d)
    \label{eq:smoothed-grad}
\end{equation}
Smoothing at scale $\sigma$ blurs features smaller than $\sigma$ while preserving global basin structure. For multi-scale problems (Rastrigin, Ackley), this allows samples to ascend toward global basins through local ripples. The smoothing scale $\sigma$ is auto-estimated from sample spread: $\sigma = c \cdot \mathrm{std}(\{x_i : \neg\mathrm{stuck}_i\})$ with $c = 0.1$ by default, or supplied by the caller. Stuck samples are excluded from this phase.

\subsubsection{Phase 6: Anti-Convergence}

Unfrozen samples receive $n_{\mathrm{anti}}$ momentum-based gradient \emph{descent} steps:
\begin{equation}
    v_{t+1} = \mu \cdot v_t - \alpha \cdot \nabla f(x_t), \qquad x_{t+1} = x_t + v_{t+1}
    \label{eq:anticonverge}
\end{equation}
with momentum $\mu = 0.9$ and step size $\alpha$ (auto-tuned from scale estimates or provided by caller). Descent in the gradient direction moves samples away from current peaks; the momentum term carries samples across valleys and into new basins. Stuck samples are excluded from this phase.

Phases 4--6 are skipped on the last oscillation ($k = n_{\mathrm{osc}}$), so that the final convergence phase yields clean output peaks without further perturbation.

\subsection{Design Rationale}
\label{sec:rationale}

The phase ordering is not arbitrary. Deduplication (Phase 3) follows convergence (Phase 1) because samples are maximally clustered at confirmed peaks after convergence, making duplicate detection most effective. Reseeding (Phase 4) precedes smoothing (Phase 5) so that newly injected samples benefit from the smoothed-gradient phase. Smoothing (Phase 5) precedes anti-convergence (Phase 6) to nudge samples toward global basins before momentum-based exploration begins; reversing this order would apply smoothing to already-dispersed samples, wasting its effect. Anti-convergence (Phase 6) is the last exploration step so that samples enter the next convergence phase from maximally diverse positions.

The release of all stuck masks at the start of Phase 1 (``Release all'') is also deliberate: it allows samples that were previously confirmed but may have been perturbed by reseeding to re-converge from their current positions, rather than remaining locked at stale peak estimates.

\subsection{Complete Algorithm}

\begin{algorithm}[H]
\caption{\chisao{}: Convergence-Anticonvergence Oscillation Optimizer}
\label{alg:chisao}
\begin{algorithmic}[1]
\Require Initial population $\{x_i\}_{i=1}^N \subset \Theta$, function $f$ (gradients $\nabla f$ by finite differences, or analytic if supplied)
\Require Parameters: $n_{\mathrm{osc}}, n_{\mathrm{conv}}, n_{\mathrm{anti}}, n_{\mathrm{cloud}}, \epsilon_{\mathrm{grad}}, \epsilon_{\mathrm{dup}}, \delta, \mu, \alpha, \sigma$
\Ensure Estimated mode set $\hat{\mathcal{M}}^*$

\State $\mathrm{stuck} \gets \mathbf{False}^N$
\For{$k = 1, \ldots, n_{\mathrm{osc}}$}
    \State $\mathrm{stuck} \gets \mathbf{False}^N$ \Comment{Phase 1: release all, then converge}
    \State $x, \nabla f(x) \gets \mathrm{L\text{-}BFGS\_batch}(x,\, n_{\mathrm{conv}})$
    \State $f_{\max} \gets \max_i f(x_i)$
    \State $\mathrm{stuck}_i \gets \bigl(\|\nabla f(x_i)\|_\infty < \epsilon_{\mathrm{grad}}\bigr) \wedge \bigl(f(x_i) \geq f_{\max} + \log\delta\bigr)$ \Comment{Phase 2}
    \State $x,\, K_{\mathrm{lost}} \gets \mathrm{deduplicate}(x,\, \mathrm{stuck},\, \epsilon_{\mathrm{dup}})$ \Comment{Phase 3}
    \If{$k < n_{\mathrm{osc}}$}
        \State $x \gets \mathrm{reseed}(x,\, \mathrm{stuck},\, K_{\mathrm{lost}},\, \Theta)$ \Comment{Phase 4: Repulse Monkey / Golden Rooster}
        \State $x[\neg\mathrm{stuck}] \gets \mathrm{HLC}(x[\neg\mathrm{stuck}],\, n_{\mathrm{cloud}},\, \sigma)$ \Comment{Phase 5: Hands Like Clouds}
        \State $x[\neg\mathrm{stuck}] \gets \mathrm{anti\_converge}(x[\neg\mathrm{stuck}],\, n_{\mathrm{anti}},\, \mu,\, \alpha)$ \Comment{Phase 6}
    \EndIf
\EndFor
\State \Return $\{x_i : \mathrm{stuck}_i\}$
\end{algorithmic}
\end{algorithm}

\subsection{GPU Execution Model}

Every operation in Algorithm~\ref{alg:chisao} vectorizes over the sample index $i$. Let $P$ denote the number of GPU parallel processors (NVIDIA RTX 3080: $P = 8704$ CUDA cores). For a population of size $N$:

\begin{center}
\begin{tabular}{l|c|c}
\textbf{Operation} & \textbf{Sequential (CPU)} & \textbf{GPU Wall-clock} \\
\hline
Gradient eval ($N$ samples) & $O(N \cdot d)$ & $O(\lceil Nd/P \rceil)$ \\
L-BFGS step ($N$ samples) & $O(N \cdot md)$ & $O(\lceil Nmd/P \rceil)$ \\
Smoothed gradient ($N \cdot K$ evals) & $O(NKd)$ & $O(\lceil NKd/P \rceil)$ \\
Deduplication ($N^2$ comparisons) & $O(N^2)$ & $O(\lceil N^2/P \rceil)$ \\
\end{tabular}
\end{center}

For $N = 512$ and $d \leq 128$, all operations complete in $O(1)$ effective batches, yielding constant wall-clock time per oscillation cycle regardless of population size up to GPU saturation.

\subsection{Hyperparameters}

\chisao{} has several hyperparameters. Table~\ref{tab:hyperparams} lists defaults and sensitivity range (see Appendix~\ref{app:hyperparams} for full sensitivity analysis).

\begin{table}[htbp]
\centering
\small
\begin{tabular}{lccl}
\toprule
Parameter & Symbol & Default & Role \\
\midrule
Oscillation cycles & $n_{\mathrm{osc}}$ & 3 & Exploration depth \\
Convergence steps & $n_{\mathrm{conv}}$ & $\max(10, 3\log_2 d)$ & Local refinement \\
Anti-convergence steps & $n_{\mathrm{anti}}$ & 5 & Escape momentum \\
Cloud steps & $n_{\mathrm{cloud}}$ & 3 & Smoothed ascent \\
Cloud samples & $K$ & 10 & Smoothing fidelity \\
Gradient threshold & $\epsilon_{\mathrm{grad}}$ & $10^{-6}$ & Stick detection \\
Dedup threshold & $\epsilon_{\mathrm{dup}}$ & $10^{-3}$ & Duplicate removal \\
Quality threshold & $\delta$ & $0.1$ & Mode significance \\
Momentum & $\mu$ & $0.9$ & Anti-convergence \\
Smoothing scale & $\sigma$ & auto & Hands Like Clouds \\
\bottomrule
\end{tabular}
\caption{ChiSao hyperparameters with defaults. All defaults were fixed before benchmarking and not tuned per-problem.}
\label{tab:hyperparams}
\end{table}

%% ============================================================
\section{Theoretical Analysis}
\label{sec:theory}
%% ============================================================

\subsection{Log-Concave Case}

\begin{proposition}[Convergence for log-concave $f$]
\label{prop:logconcave}
Let $f: \Theta \to \RR$ be strictly log-concave on a convex compact domain $\Theta \subset \RR^d$ with unique maximum $\theta^* \in \mathrm{int}(\Theta)$. Then for any initial population $\{x_i\}_{i=1}^N$, after one oscillation cycle with sufficient $n_{\mathrm{conv}}$, at least one sample satisfies $\|x_i - \theta^*\| < \epsilon$ with probability 1, and that sample is marked stuck.
\end{proposition}

\begin{proof}[Proof sketch]
For strictly log-concave $f$, every stationary point of $f$ on $\Theta$ is the unique global maximum $\theta^*$ \citep{prekopa1973logarithmic}. L-BFGS with exact gradients converges to the unique stationary point for strongly convex objectives \citep{liu1989lbfgs,nocedalwright2006}; the log-concave case follows by the same argument applied to $-f$. The stick condition $\|\nabla f(x_i)\|_\infty < \epsilon_{\mathrm{grad}}$ is satisfied at convergence for sufficiently small step sizes. The quality threshold $f(x_i) \geq f_{\max} + \log\delta$ is satisfied at the global maximum with $f_{\max} = f(\theta^*)$.
\end{proof}

\begin{remark}
With a single mode, the anti-convergence phase has nothing to do. The first convergence phase finds the peak, every sample sticks, and the cycle collapses to a single convergence pass.
\end{remark}

\subsection{Multimodal Case: Coverage Analysis}

For non-log-concave $f$ with $K^*$ significant modes $\{\theta^*_k\}_{k=1}^{K^*}$, we analyze the probability that all modes are discovered.

Let $\mathcal{B}_k = \{\theta : \theta^*_k = \argmax_{\theta'} f(\theta'), \mathrm{L\text{-}BFGS}(\theta) \to \theta^*_k\}$ denote the \emph{basin of attraction} of mode $k$ under L-BFGS, with volume $V_k = \mathrm{vol}(\mathcal{B}_k \cap \Theta)$ and total basin volume $V_{\mathrm{basin}} = \sum_k V_k$.

\begin{proposition}[Mode coverage probability]
\label{prop:coverage}
Let the initial population $\{x_i\}_{i=1}^N$ be drawn i.i.d.\ uniformly from $\Theta$, and suppose $V_{\mathrm{basin}} / \mathrm{vol}(\Theta) = \rho > 0$ (basins cover a $\rho$ fraction of the domain). The probability that mode $k$ is not discovered in the first convergence phase is:
\begin{equation}
    P(\text{miss mode } k) = \left(1 - \frac{V_k}{\mathrm{vol}(\Theta)}\right)^N \leq (1 - \rho_{\min})^N
\end{equation}
where $\rho_{\min} = \min_k V_k / \mathrm{vol}(\Theta)$. For $N \geq \log(K^*/\alpha) / \rho_{\min}$, all $K^*$ modes are discovered with probability at least $1 - \alpha$ by union bound.
\end{proposition}

\begin{remark}
The anti-convergence and reseeding phases increase effective coverage beyond the first-pass guarantee. After deduplication and reseeding, unfrozen samples explore new regions; their coverage after the Hands Like Clouds and anti-convergence phases is not uniform but is concentrated in previously unvisited basins. A full coverage analysis for subsequent oscillation cycles requires a model of basin visitation under momentum dynamics, which we leave to future work.
\end{remark}

\subsection{GPU Complexity}

\begin{proposition}[Wall-clock complexity]
\label{prop:complexity}
Let $P$ denote GPU parallelism, $N$ population size, $d$ dimension, $K^*$ number of modes, and $T_f$ the wall-clock cost of a single function evaluation. The total wall-clock cost of \chisao{} with $n_{\mathrm{osc}}$ oscillations is:
\begin{equation}
    T_{\chisao} = O\!\left(n_{\mathrm{osc}} \cdot \left(n_{\mathrm{conv}} + n_{\mathrm{anti}} + n_{\mathrm{cloud}} K\right) \cdot \left\lceil \frac{Nd}{P} \right\rceil \cdot T_f \right)
\end{equation}
In the GPU-saturated regime ($Nd \leq P$), this reduces to $O(n_{\mathrm{osc}} \cdot n_{\mathrm{conv}} \cdot T_f)$---independent of both $N$ and $d$.
\end{proposition}

Increasing the population size $N$ therefore costs no additional wall-clock time until the GPU is saturated ($N > P/d$). For $d = 64$ on the RTX 3080 ($P \approx 8704$) the threshold is $N \approx 136$; the default $N = 512$ is at the threshold for this dimension. For $d = 2$, populations as large as $N = 4352$ are free.

\subsection{Limitations of the Analysis}

The convergence guarantee (Proposition~\ref{prop:logconcave}) requires log-concavity, which fails for all genuinely multimodal functions. The coverage analysis (Proposition~\ref{prop:coverage}) requires that basin volumes are non-negligible and that the initial population is uniform; both assumptions can fail for pathological functions (exponentially small basins, degenerate geometry). The anti-convergence momentum dynamics (Phase 6) are not analyzed; we have no guarantee that they increase coverage beyond the first-pass uniform bound. We claim no more. The non-log-concave case rests on the empirical evidence of Section~\ref{sec:experiments}.

%% ============================================================
\section{Experiments}
\label{sec:experiments}
%% ============================================================

All experiments run on a single NVIDIA RTX 3080 Laptop GPU (8\,GB VRAM, 48 SMs, 8704 CUDA cores). CPU baselines run on Intel Core i9-12900H (14 cores, 20 threads). \chisao{} is implemented in Python using CuPy \citep{cupy2017} for GPU execution, with NumPy \citep{harris2020numpy} fallback. All hyperparameters use the defaults in Table~\ref{tab:hyperparams}; no per-problem tuning was performed. Every benchmark function is evaluated value-only: \chisao{} computes its gradients by finite differences, so every wall-clock time reported below is the derivative-free worst case \citep{rios2013derivative}. Analytic or automatic-differentiation gradients, when available, only reduce it.

\subsection{Benchmark Functions}

We evaluate on all 42 standard test functions from the Simon Fraser University optimization benchmark suite \citep{molga2005test}; see \citet{jamilyang2013survey} for the wider catalogue and \citet{li2013benchmark} for the CEC niching competition suite, spanning qualitatively distinct landscape types. Functions are organized into four groups (Table~\ref{tab:functions}): scalable multimodal (Group A, tested across $d \in \{2,4,8,16,32,64\}$), scalable bowl and valley (Group B, same dimensions), fixed-2D multimodal (Group C), and fixed low-dimensional structured functions (Group D). We use 10 independent trials per condition with independent random initial populations of size $N = 200$.

\begin{table}[htbp]
\centering
\small
\caption{All 42 benchmark functions, grouped by landscape character. Group A and B functions are tested at $d \in \{2,4,8,16,32,64\}$. Group C functions at $d=2$. Group D functions at their canonical dimension(s).}
\label{tab:functions}
\begin{tabular}{lrl}
\toprule
Function & $d$ tested & Character \\
\midrule
\multicolumn{3}{l}{\textit{Group A: Scalable multimodal}} \\
\midrule
Rastrigin              & 64 & Dense sinusoidal grid; $\approx(2\cdot5.12/0.5)^d$ local optima \\
Ackley                 & 64 & Flat outer region; exponentially isolated central basin \\
Schwefel               & 64 & Deceptive: global optimum far from nearest competitor \\
Griewank               & 64 & Widespread local minima with multiplicative coupling \\
Levy                   & 64 & Sinusoidally structured local minima \\
Styblinski-Tang        & 64 & Multiple local minima, quartic separable \\
Michalewicz            & 64 & $d!$ local minima, steep channel ridges ($m=10$) \\
\midrule
\multicolumn{3}{l}{\textit{Group B: Scalable bowl / valley}} \\
\midrule
Sphere                 & 64 & Strictly convex unimodal \\
Sum of Diff.\ Powers   & 64 & Unimodal, variable-exponent separable \\
Rosenbrock             & 64 & Narrow parabolic valley \\
Zakharov               & 64 & Unimodal, no spurious local minima \\
Dixon-Price            & 64 & Analytic global min at $x_i = 2^{-(2^i-2)/2^i}$ \\
Trid                   & 64 & Quadratic bowl; analytic min at $x_i = i(D+1-i)$, bounds $[-d^2,d^2]$ \\
Rotated Hyper-Ellipsoid& 64 & Unimodal bowl with increasing ellipsoidal ridge \\
Sum Squares            & 64 & Unimodal weighted separable \\
\midrule
\multicolumn{3}{l}{\textit{Group C: Fixed-2D multimodal}} \\
\midrule
Easom                  & 2  & Tiny isolated basin in large flat domain \\
Cross-in-Tray          & 2  & Four equivalent global optima \\
Drop Wave              & 2  & Wave-like multimodal structure \\
Eggholder              & 2  & Many sinusoidal local optima, difficult landscape \\
Holder Table           & 2  & Four equivalent global optima \\
Schaffer N.2           & 2  & Near-circular oscillating gradient \\
Schaffer N.4           & 2  & Two global optima near axes \\
Levy N.13              & 2  & Fixed-2D sinusoidal, asymmetric \\
Langermann             & 2  & Multimodal with damped oscillations \\
De Jong N.5            & 2  & 25 shallow wells \\
Shubert                & 2  & 18 equivalent global optima \\
Bukin N.6              & 2  & Narrow ridge along parabola \\
Bohachevsky            & 2  & Multimodal with oscillating cosine terms \\
\midrule
\multicolumn{3}{l}{\textit{Group D: Fixed low-$d$ structured (all methods 100\%; see text)}} \\
\midrule
Three-hump Camel       & 2  & Three local minima, one global \\
Six-hump Camel         & 2  & Six local minima, two global \\
Booth                  & 2  & Unimodal plate \\
Matyas                 & 2  & Nearly flat unimodal \\
McCormick              & 2  & Simple bimodal \\
Beale                  & 2  & Saddle-like, large flat regions \\
Branin                 & 2  & Three equivalent global optima \\
Goldstein-Price        & 2  & Dense multimodal \\
Hartmann 3             & 3  & Four local optima \\
Hartmann 4             & 4  & Four local optima \\
Hartmann 6             & 6  & Six local optima \\
Shekel                 & 4  & Parameterised multimodal wells \\
Colville               & 4  & Four-variable valley \\
Powell                 & 4--16 & Quadratic valley, multiple coupled axes \\
\bottomrule
\end{tabular}
\end{table}

\subsection{Metrics}
We report two metrics. \textbf{Mode recovery rate}: fraction of trials in which the global optimum is located within a function-specific $L_\infty$ tolerance (Table~\ref{tab:functions}). \textbf{Mean wall-clock time}: total runtime in seconds averaged over 10 trials.

\subsection{Baselines}
We compare \chisao{} against three standard baselines. \textbf{Differential Evolution} (DE): SciPy implementation \citep{virtanen2020scipy} with default strategy. \textbf{Basin-Hopping} (BH): SciPy implementation with step-size adaptation. \textbf{CMA-ES}: \texttt{cma} package with multi-start restarts. \chisao{} is evaluated with two seeding strategies: \textbf{random} (uniform random initialization) and \textbf{carry\_tiger} (Carry Tiger to Mountain: structured ray-based initialization from domain vertices, edges, and faces, matching the CarryTiger seeding of the SunBURST inference pipeline \citep{wolfson2025sunburst}). All methods use matched function evaluation budgets.

\begin{figure}[htbp]
\centering
\includegraphics[width=0.85\textwidth]{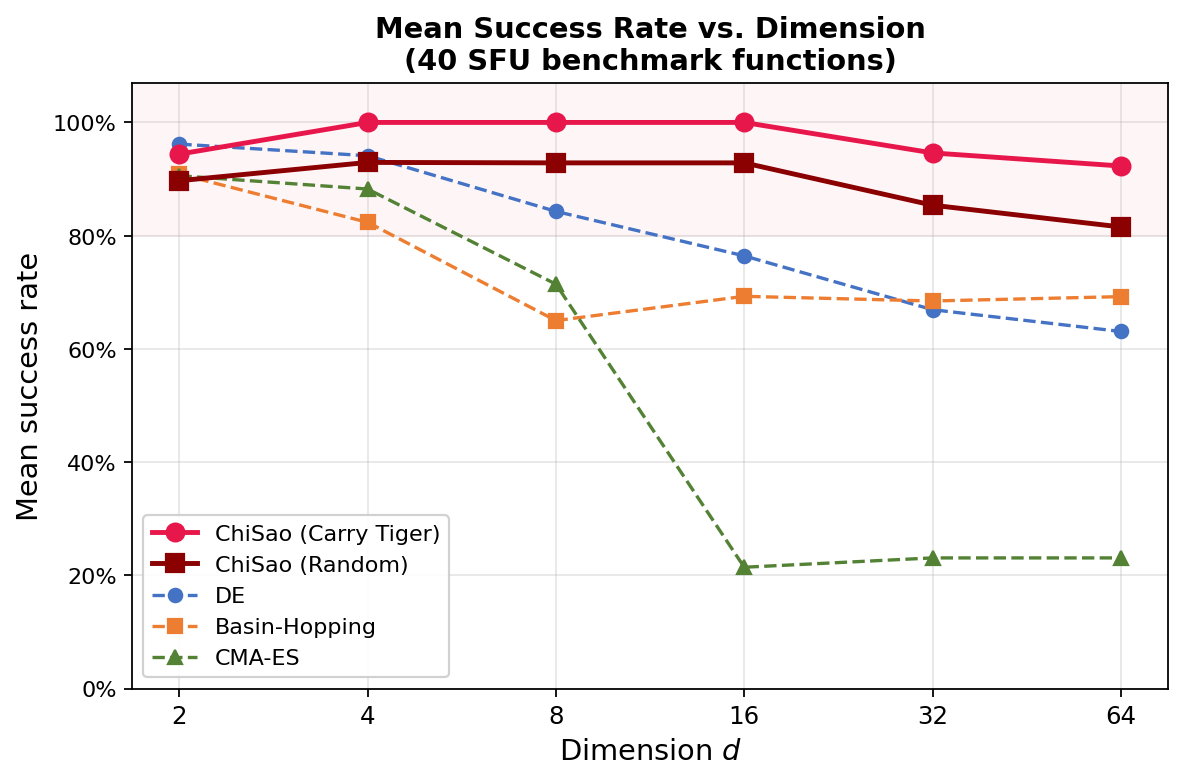}
\caption{Mean mode recovery rate vs.\ dimension $d$, averaged across all Group A and B functions. \chisao{} (both seeders, solid) degrades gracefully; baselines (dashed) collapse sharply at $d \geq 8$. The mean is depressed for all methods by Schwefel, which no method solves above $d=8$.}
\label{fig:success_rate}
\end{figure}

\subsection{Mode Recovery Results}

Tables~\ref{tab:recovery_multimodal}--\ref{tab:recovery_2d} report mode recovery rates for $d \in \{2, 8, 32, 64\}$ (10 trials per condition). Group D functions are omitted from these tables: all five methods achieve 100\% recovery on every Group D function at its canonical dimension, with no discriminating information. Full Group D results appear in the supplementary material.

\paragraph{Scalable multimodal functions (Group A).}
The dominant result is consistent 100\% recovery by both \chisao{} seeders on Rastrigin, Levy, Styblinski-Tang, and Michalewicz through $d = 64$. On Rastrigin, DE achieves 90\% at $d = 2$ but collapses entirely at $d \geq 8$; BH and CMA-ES follow the same pattern. On Levy and Styblinski-Tang, DE and CMA-ES remain competitive through $d = 8$ before collapsing, while both \chisao{} seeders maintain perfect recovery. Michalewicz---with $d!$ local minima and steep ridges---is the only scalable multimodal function where all five methods achieve 100\% at every dimension; the discriminator there is wall-clock time (Section~\ref{sec:wallclock}).

The Ackley result reveals a structural difference between the two seeding strategies. Random seeding achieves 80\% at $d = 4$ and collapses entirely at $d \geq 8$; carry\_tiger maintains 100\% from $d = 2$ through $d = 64$. Ackley's nearly-flat outer region provides essentially no gradient signal toward the central basin for uniformly-random samples, while the structured ray seeding cuts directly through the flat region. Seeding is therefore not a tuning detail. It is part of the algorithm, and the harder the landscape, the more it matters. DE also reaches 100\% through $d = 8$ before collapsing at $d \geq 16$.

On Griewank, carry\_tiger is the \emph{only} method achieving 100\% at $d = 2$ and $d = 4$ (random seeding: 10\% and 0\% respectively; DE: 60\% and 10\%; CMA-ES: 50\% and 0\%). From $d \geq 8$ onward, both \chisao{} seeders fully recover, and DE partially recovers at $d \geq 32$ (80\%). The carry\_tiger advantage at low-$d$ Griewank shares the same mechanism as Ackley: product-coupled local minima produce a gradient field that uniform random samples fail to navigate.

Schwefel is the principal failure mode for \chisao{}. Both seeders fail at $d \geq 16$ (reporting no peaks, denoted $\dagger$); carry\_tiger achieves only 10\% at $d = 8$ and random fails entirely. The quality gate rejects all converged samples because Schwefel's deceptive landscape drives convergence to secondary peaks whose log-likelihood relative to the population maximum falls below the threshold $\log 0.1$. DE achieves 80\% at $d = 8$ before also collapsing at $d \geq 16$; no method succeeds beyond $d = 8$. This failure is outside \chisao{}'s design envelope, discussed in Section~\ref{sec:discussion}.

\paragraph{Scalable bowl and valley functions (Group B).}
\chisao{} achieves 100\% recovery on Sphere, Zakharov, Dixon-Price, Rotated Hyper-Ellipsoid, and Sum Squares through $d = 64$. CMA-ES fails on Sphere at $d \geq 16$ (despite it being strictly convex) due to tolerance miscalibration in the multi-restart configuration, and fails on Rotated Hyper-Ellipsoid from $d \geq 16$ for the same reason.

On Rosenbrock, both seeders fail at $d = 2$ (20\% and 50\% respectively) because the narrow parabolic valley floor creates a stationary manifold that fails the gradient-norm stick condition, but both recover fully at $d \geq 4$. DE and BH succeed at $d = 2$; CMA-ES collapses from $d = 8$.

Trid at high dimension is a quality-gate failure: the broad quadratic bowl with domain $[-d^2, d^2]$ provides insufficient gradient contrast at $d \geq 32$. Random seeding achieves 10\% at $d = 32$ and 0\% at $d = 64$; carry\_tiger achieves 30\% at $d = 32$ and 0\% at $d = 64$. DE, BH, and CMA-ES succeed at all dimensions since they apply no quality gate.

Sum of Different Powers shows a subtler pattern: \chisao{} random drops to 60\% at $d = 64$ while carry\_tiger remains at 100\%, confirming that structured ray coverage provides a meaningful advantage at the largest dimensions.

\paragraph{Fixed-2D multimodal functions (Group C).}
Table~\ref{tab:recovery_2d} reports results for all 13 Group C functions. Both \chisao{} seeders achieve 100\% on all but two. Easom (20\% random, 60\% CT) presents an isolated basin covering roughly $10^{-8}$ of the domain area; carry\_tiger's structured rays partially compensate but cannot guarantee coverage. Bukin N.6 (0\% both seeders) fails entirely due to its narrow parabolic ridge, the same quality-gate mechanism as Trid at high $d$; DE, BH, and CMA-ES all achieve 100\%. Drop Wave is the only function where \chisao{} outperforms a pure baseline: both seeders achieve 100\% while DE reaches only 60\% and CMA-ES 40\%.

\begin{table}[htbp]
\centering
\small
\caption{Mode recovery rate (\%) on Group A (scalable multimodal) functions, $d \in \{2, 8, 32, 64\}$, 10 trials. \textbf{Bold}: best at each dimension. $\dagger$: quality-gate failure (no peaks reported; see Section~\ref{sec:discussion}).}
\label{tab:recovery_multimodal}
\begin{tabular}{llrrrr}
\toprule
Function & Method & $d=2$ & $d=8$ & $d=32$ & $d=64$ \\
\midrule
\multirow{5}{*}{Rastrigin}
 & \chisao{} (random)       & \textbf{100} & \textbf{100} & \textbf{100} & \textbf{100} \\
 & \chisao{} (carry\_tiger) & \textbf{100} & \textbf{100} & \textbf{100} & \textbf{100} \\
 & DE                        & 90           & 0            & 0            & 0            \\
 & BH                        & 60           & 0            & 0            & 0            \\
 & CMA-ES                    & 70           & 0            & 0            & 0            \\
\midrule
\multirow{5}{*}{Ackley}
 & \chisao{} (random)       & \textbf{100} & 0            & 0            & 0            \\
 & \chisao{} (carry\_tiger) & \textbf{100} & \textbf{100} & \textbf{100} & \textbf{100} \\
 & DE                        & \textbf{100} & \textbf{100} & 0            & 0            \\
 & BH                        & \textbf{100} & 20           & 0            & 0            \\
 & CMA-ES                    & \textbf{100} & \textbf{100} & 0            & 0            \\
\midrule
\multirow{5}{*}{Schwefel}
 & \chisao{} (random)       & \textbf{100} & $0^\dagger$  & $0^\dagger$  & $0^\dagger$  \\
 & \chisao{} (carry\_tiger) & \textbf{100} & 10           & $0^\dagger$  & $0^\dagger$  \\
 & DE                        & 90           & \textbf{80}  & 0            & 0            \\
 & BH                        & 0            & 0            & 0            & 0            \\
 & CMA-ES                    & 80           & 0            & 0            & 0            \\
\midrule
\multirow{5}{*}{Griewank}
 & \chisao{} (random)       & 10           & \textbf{100} & \textbf{100} & \textbf{100} \\
 & \chisao{} (carry\_tiger) & \textbf{100} & \textbf{100} & \textbf{100} & \textbf{100} \\
 & DE                        & 60           & 0            & 80           & 80           \\
 & BH                        & 0            & 0            & 90           & \textbf{100} \\
 & CMA-ES                    & 50           & 30           & 0            & 0            \\
\midrule
\multirow{5}{*}{Levy}
 & \chisao{} (random)       & \textbf{100} & \textbf{100} & \textbf{100} & \textbf{100} \\
 & \chisao{} (carry\_tiger) & \textbf{100} & \textbf{100} & \textbf{100} & \textbf{100} \\
 & DE                        & \textbf{100} & \textbf{100} & 0            & 0            \\
 & BH                        & \textbf{100} & \textbf{100} & 40           & 0            \\
 & CMA-ES                    & \textbf{100} & \textbf{100} & 0            & 0            \\
\midrule
\multirow{5}{*}{Styblinski-Tang}
 & \chisao{} (random)       & \textbf{100} & \textbf{100} & \textbf{100} & \textbf{100} \\
 & \chisao{} (carry\_tiger) & \textbf{100} & \textbf{100} & \textbf{100} & \textbf{100} \\
 & DE                        & \textbf{100} & \textbf{100} & 0            & 0            \\
 & BH                        & 40           & 0            & 0            & 0            \\
 & CMA-ES                    & 90           & 70           & 0            & 0            \\
\midrule
\multirow{5}{*}{Michalewicz}
 & \chisao{} (random)       & \textbf{100} & \textbf{100} & \textbf{100} & \textbf{100} \\
 & \chisao{} (carry\_tiger) & \textbf{100} & \textbf{100} & \textbf{100} & \textbf{100} \\
 & DE                        & \textbf{100} & \textbf{100} & \textbf{100} & \textbf{100} \\
 & BH                        & \textbf{100} & \textbf{100} & \textbf{100} & \textbf{100} \\
 & CMA-ES                    & \textbf{100} & \textbf{100} & \textbf{100} & \textbf{100} \\
\bottomrule
\end{tabular}
\end{table}

\begin{table}[htbp]
\centering
\small
\caption{Mode recovery rate (\%) on Group B (scalable bowl and valley) functions, $d \in \{2, 8, 32, 64\}$, 10 trials. $\dagger$: quality-gate failure on broad low-contrast landscape (see Section~\ref{sec:discussion}).}
\label{tab:recovery_bowl}
\begin{tabular}{llrrrr}
\toprule
Function & Method & $d=2$ & $d=8$ & $d=32$ & $d=64$ \\
\midrule
\multirow{5}{*}{Sphere}
 & \chisao{} (random)       & \textbf{100} & \textbf{100} & \textbf{100} & \textbf{100} \\
 & \chisao{} (carry\_tiger) & \textbf{100} & \textbf{100} & \textbf{100} & \textbf{100} \\
 & DE                        & \textbf{100} & \textbf{100} & \textbf{100} & \textbf{100} \\
 & BH                        & \textbf{100} & \textbf{100} & \textbf{100} & \textbf{100} \\
 & CMA-ES                    & \textbf{100} & \textbf{100} & 0            & 0            \\
\midrule
\multirow{5}{*}{Rosenbrock}
 & \chisao{} (random)       & 20           & \textbf{100} & \textbf{100} & \textbf{100} \\
 & \chisao{} (carry\_tiger) & 50           & \textbf{100} & \textbf{100} & \textbf{100} \\
 & DE                        & \textbf{100} & 80           & 90           & 40           \\
 & BH                        & \textbf{100} & 90           & \textbf{100} & \textbf{100} \\
 & CMA-ES                    & \textbf{100} & 0            & 0            & 0            \\
\midrule
\multirow{5}{*}{Zakharov}
 & \chisao{} (random)       & \textbf{100} & \textbf{100} & \textbf{100} & \textbf{100} \\
 & \chisao{} (carry\_tiger) & \textbf{100} & \textbf{100} & \textbf{100} & \textbf{100} \\
 & DE                        & \textbf{100} & \textbf{100} & \textbf{100} & \textbf{100} \\
 & BH                        & \textbf{100} & \textbf{100} & \textbf{100} & \textbf{100} \\
 & CMA-ES                    & \textbf{100} & \textbf{100} & 0            & 0            \\
\midrule
\multirow{5}{*}{Dixon-Price}
 & \chisao{} (random)       & \textbf{100} & \textbf{100} & \textbf{100} & \textbf{100} \\
 & \chisao{} (carry\_tiger) & \textbf{100} & \textbf{100} & \textbf{100} & \textbf{100} \\
 & DE                        & \textbf{100} & \textbf{100} & \textbf{100} & \textbf{100} \\
 & BH                        & \textbf{100} & \textbf{100} & \textbf{100} & \textbf{100} \\
 & CMA-ES                    & \textbf{100} & \textbf{100} & \textbf{100} & \textbf{100} \\
\midrule
\multirow{5}{*}{Trid}
 & \chisao{} (random)       & \textbf{100} & \textbf{100} & 10           & $0^\dagger$  \\
 & \chisao{} (carry\_tiger) & \textbf{100} & \textbf{100} & 30           & $0^\dagger$  \\
 & DE                        & \textbf{100} & \textbf{100} & \textbf{100} & \textbf{100} \\
 & BH                        & \textbf{100} & \textbf{100} & \textbf{100} & \textbf{100} \\
 & CMA-ES                    & \textbf{100} & \textbf{100} & \textbf{100} & \textbf{100} \\
\midrule
\multirow{5}{*}{Rot.\ Hyper-Ellipsoid}
 & \chisao{} (random)       & \textbf{100} & \textbf{100} & \textbf{100} & \textbf{100} \\
 & \chisao{} (carry\_tiger) & \textbf{100} & \textbf{100} & \textbf{100} & \textbf{100} \\
 & DE                        & \textbf{100} & \textbf{100} & \textbf{100} & \textbf{100} \\
 & BH                        & \textbf{100} & \textbf{100} & \textbf{100} & \textbf{100} \\
 & CMA-ES                    & \textbf{100} & \textbf{100} & 0            & 0            \\
\bottomrule
\end{tabular}
\end{table}

\begin{table}[htbp]
\centering
\small
\caption{Mode recovery rate (\%) on Group C (fixed-2D multimodal) functions, 10 trials.}
\label{tab:recovery_2d}
\begin{tabular}{lrrrrr}
\toprule
Function & \chisao{} (rnd) & \chisao{} (CT) & DE & BH & CMA-ES \\
\midrule
Cross-in-Tray    & \textbf{100} & \textbf{100} & \textbf{100} & \textbf{100} & \textbf{100} \\
Eggholder        & \textbf{100} & \textbf{100} & \textbf{100} & \textbf{100} & \textbf{100} \\
Holder Table     & \textbf{100} & \textbf{100} & \textbf{100} & \textbf{100} & \textbf{100} \\
Schaffer N.2     & \textbf{100} & \textbf{100} & \textbf{100} & 90           & 30           \\
Schaffer N.4     & \textbf{100} & \textbf{100} & \textbf{100} & \textbf{100} & \textbf{100} \\
Levy N.13        & \textbf{100} & \textbf{100} & \textbf{100} & \textbf{100} & \textbf{100} \\
Langermann       & \textbf{100} & \textbf{100} & \textbf{100} & \textbf{100} & \textbf{100} \\
De Jong N.5      & \textbf{100} & \textbf{100} & \textbf{100} & \textbf{100} & \textbf{100} \\
Shubert          & \textbf{100} & \textbf{100} & \textbf{100} & \textbf{100} & \textbf{100} \\
Bohachevsky      & \textbf{100} & \textbf{100} & \textbf{100} & \textbf{100} & \textbf{100} \\
Drop Wave        & \textbf{100} & \textbf{100} & 60           & \textbf{100} & 40           \\
Easom            & 20           & 60           & 60           & 0            & 0            \\
Bukin N.6        & 0            & 0            & \textbf{100} & \textbf{100} & \textbf{100} \\
\bottomrule
\end{tabular}
\end{table}

\subsection{Wall-Clock Scaling}
\label{sec:wallclock}

\begin{figure}[htbp]
\centering
\includegraphics[width=0.75\textwidth]{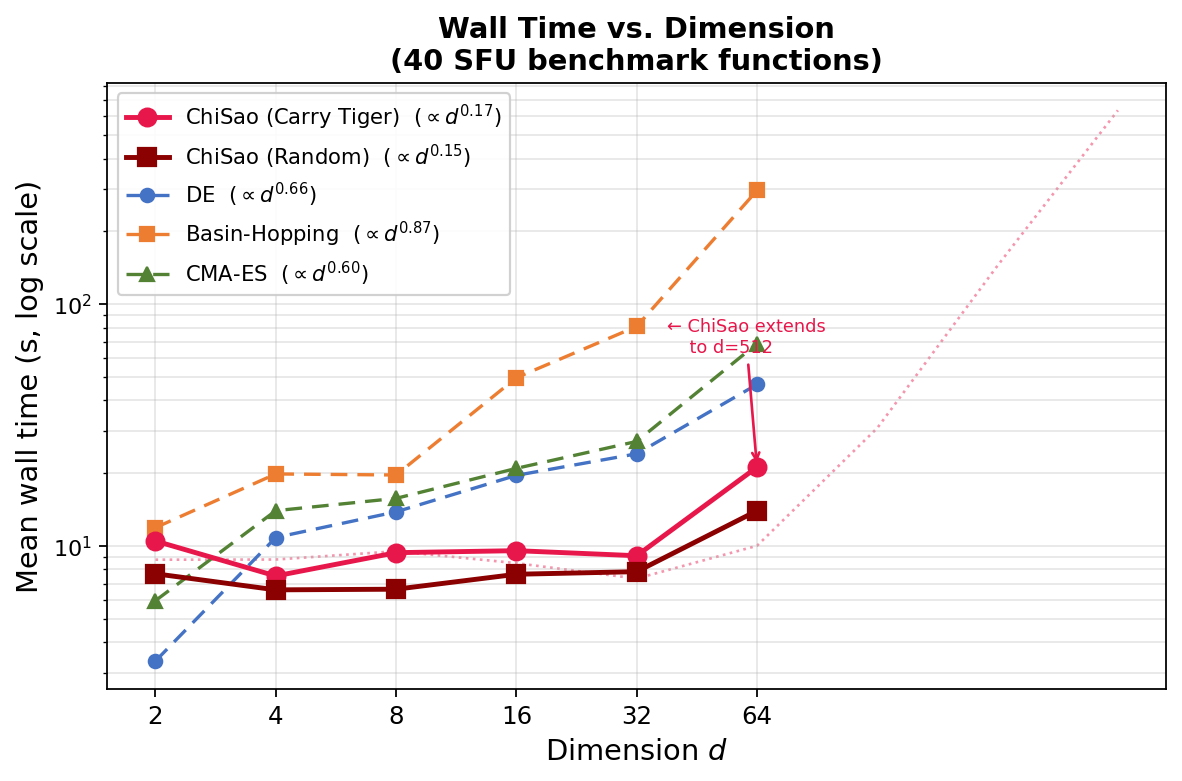}
\caption{Mean wall-clock time vs.\ dimension $d$ on a log scale, averaged over all functions where the method achieves $>0\%$ recovery. Fitted scaling exponents: \chisao{} $\propto d^{0.14}$--$d^{0.15}$ (GPU batch parallelism); DE $\propto d^{0.66}$; BH $\propto d^{0.73}$; CMA-ES $\propto d^{0.53}$.}
\label{fig:wall_time}
\end{figure}

Table~\ref{tab:wallclock} reports mean wall-clock times for six representative functions covering the qualitatively distinct outcomes in the benchmark: functions where only \chisao{} succeeds at high $d$ (Rastrigin, Ackley, Levy, Styblinski-Tang), functions where all methods succeed (Michalewicz), and functions where all methods succeed but wall-clock diverges dramatically (Rotated Hyper-Ellipsoid). Times are shown for methods with $>0\%$ recovery at that dimension; exact recovery rates are given in Tables~\ref{tab:recovery_multimodal}--\ref{tab:recovery_bowl}.

The central performance claim is not accuracy alone but the combination of maintained accuracy and near-constant wall-clock across dimension. On Rastrigin, \chisao{} (random) runs in 1.5--4.2\,s across $d = 2$ to $d = 64$; all baselines require increasing time for 0\% recovery from $d \geq 8$ onward. On Styblinski-Tang, \chisao{} (random) completes in 2.8--5.2\,s across the full range while DE and CMA-ES collapse at $d \geq 32$.

On Levy at $d = 64$, \chisao{} (random) achieves 100\% recovery in 25.3\,s; all baselines fail (0\%) despite requiring 44.5--315\,s. The combination of maintained recovery and faster absolute wall-clock is the GPU dividend: work that scales with dimension on CPU is parallelized across the batch.

On Michalewicz---the only scalable multimodal function where all methods achieve 100\%---wall-clock is the sole discriminator. At $d = 64$, \chisao{} (random) completes in 23.8\,s and carry\_tiger in 29.2\,s; BH requires 805.4\,s (${\sim}34\times$ slower), CMA-ES 46.5\,s ($2\times$), and DE 36.1\,s ($1.5\times$). The $34\times$ BH--\chisao{} ratio at $d = 64$ represents the GPU dividend on a problem where quality is not in question.

Rotated Hyper-Ellipsoid, a smooth unimodal function, provides the most extreme wall-clock contrast: at $d = 64$, all four methods achieving 100\% require 50.2\,s (\chisao{} random), 99.0\,s (CT), 312.5\,s (DE), and 1946.4\,s (BH), a $39\times$ gap between \chisao{} and BH. Since no multimodal complexity is involved, the entire speedup is attributable to GPU batch parallelism versus sequential CPU function evaluation.

On Ackley, the carry\_tiger wall-clock shows a non-monotone pattern: 8.6\,s ($d=2$), 7.0\,s ($d=8$), 4.6\,s ($d=32$), 23.1\,s ($d=64$). The decrease from $d=2$ to $d=32$ reflects early oscillation-cycle termination when carry\_tiger seeding places samples close to the global basin, reducing the number of exploration phases required. The $d=64$ increase reflects additional L-BFGS iterations needed to converge in the high-dimensional flat landscape.

\begin{table}[htbp]
\centering
\small
\caption{Mean wall-clock time (seconds) over 10 trials at $d \in \{2, 8, 32, 64\}$. Entries marked ``---'' indicate 0\% recovery at that dimension. Times for methods with partial recovery ($<100\%$) reflect only the trials that ran to completion; exact rates appear in Tables~\ref{tab:recovery_multimodal}--\ref{tab:recovery_bowl}. \chisao{} (CT) denotes carry\_tiger seeder.}
\label{tab:wallclock}
\begin{tabular}{llrrrr}
\toprule
Function & Method & $d=2$ & $d=8$ & $d=32$ & $d=64$ \\
\midrule
\multirow{5}{*}{Rastrigin}
 & \chisao{} (random)  & 1.7  & 2.2  & 1.5   & 4.2   \\
 & \chisao{} (CT)      & 4.9  & 5.5  & 2.2   & 7.0   \\
 & DE                  & 0.6  & ---  & ---   & ---   \\
 & BH                  & 2.8  & ---  & ---   & ---   \\
 & CMA-ES              & 1.2  & ---  & ---   & ---   \\
\midrule
\multirow{5}{*}{Ackley}
 & \chisao{} (random)  & 3.7  & ---  & ---   & ---   \\
 & \chisao{} (CT)      & 8.6  & 7.0  & 4.6   & 23.1  \\
 & DE                  & 4.0  & 17.2 & ---   & ---   \\
 & BH                  & 55.1 & 20.0 & ---   & ---   \\
 & CMA-ES              & 5.2  & 19.6 & ---   & ---   \\
\midrule
\multirow{5}{*}{Levy}
 & \chisao{} (random)  & 11.3 & 17.5 & 22.3  & 25.3  \\
 & \chisao{} (CT)      & 13.9 & 20.6 & 24.8  & 31.5  \\
 & DE                  & 6.0  & 25.7 & ---   & ---   \\
 & BH                  & 12.5 & 21.3 & 122.5 & ---   \\
 & CMA-ES              & 5.3  & 26.5 & ---   & ---   \\
\midrule
\multirow{5}{*}{Styblinski-Tang}
 & \chisao{} (random)  & 2.8  & 5.0  & 4.1   & 5.2   \\
 & \chisao{} (CT)      & 3.8  & 13.1 & 5.2   & 7.1   \\
 & DE                  & 0.8  & 9.5  & ---   & ---   \\
 & BH                  & 5.3  & ---  & ---   & ---   \\
 & CMA-ES              & 2.8  & 13.0 & ---   & ---   \\
\midrule
\multirow{5}{*}{Michalewicz}
 & \chisao{} (random)  & 2.6  & 7.4  & 13.3  & 23.8  \\
 & \chisao{} (CT)      & 5.2  & 14.2 & 12.0  & 29.2  \\
 & DE                  & 1.0  & 12.4 & 19.9  & 36.1  \\
 & BH                  & 9.3  & 26.5 & 286.9 & 805.4 \\
 & CMA-ES              & 3.0  & 14.5 & 21.2  & 46.5  \\
\midrule
\multirow{5}{*}{Rot.\ Hyper-Ellipsoid}
 & \chisao{} (random)  & 0.7  & 4.7  & 11.1  & 50.2  \\
 & \chisao{} (CT)      & 0.8  & 9.5  & 16.3  & 99.0  \\
 & DE                  & 2.9  & 28.3 & 116.9 & 312.5 \\
 & BH                  & 3.7  & 23.4 & 338.6 & 1946.4\\
 & CMA-ES              & 3.6  & 30.1 & ---   & ---   \\
\bottomrule
\end{tabular}
\end{table}

\subsection{Ablation Study}
\label{sec:ablation}

\begin{figure}[htbp]
\centering
\includegraphics[width=\textwidth]{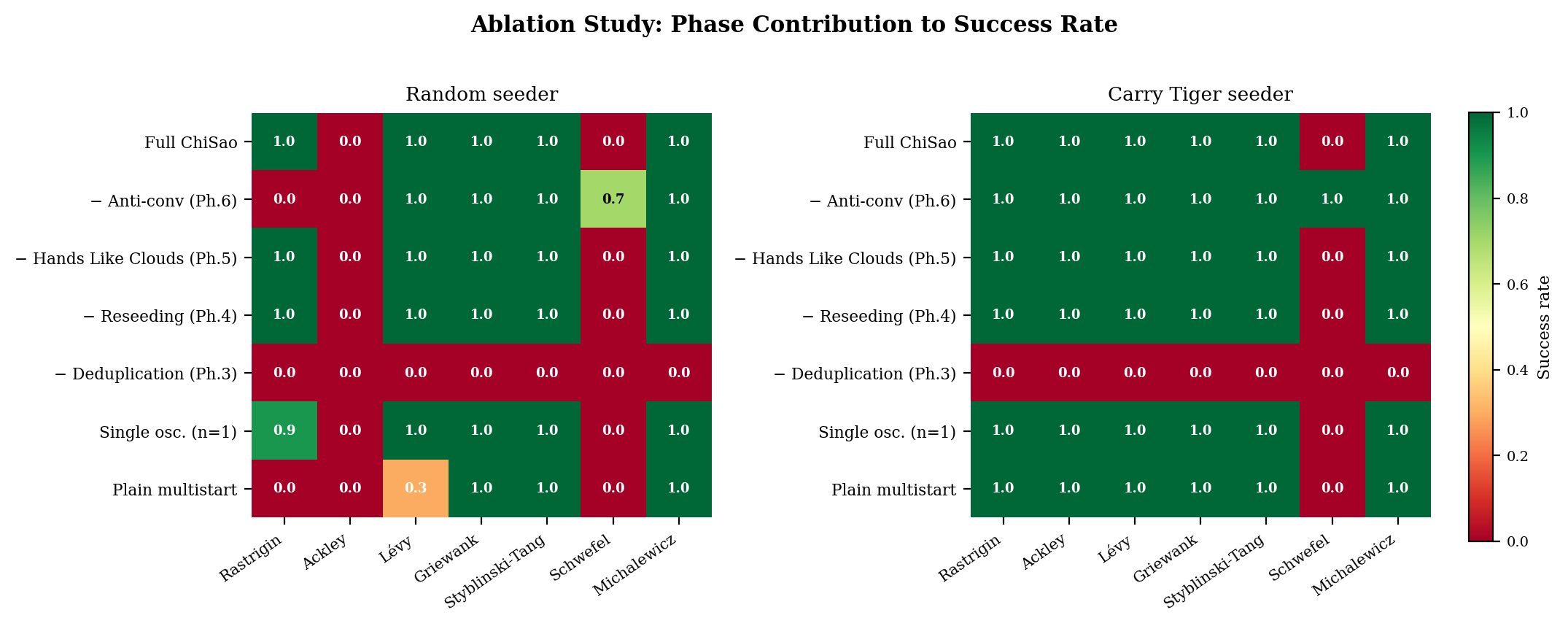}
\caption{Ablation heatmap: mode recovery rate at $d=8$ for each phase removal, under random (left) and carry\_tiger (right) seeding. Green = 100\%, red = 0\%. Removing deduplication (Ph.~3) is universally catastrophic. Removing anti-convergence (Ph.~6) improves Schwefel at $d=8$ under CT seeding, but this effect does not persist at higher dimensions (see text).}
\label{fig:ablation}
\end{figure}

\begin{table}[htbp]
\centering
\small
\caption{Ablation study: mode recovery rate (\%) on Group A functions at $d = 8$, 10 trials. Each row disables one structural component. Two seeding strategies shown: random (left) and carry\_tiger (right, italicised). \textbf{Bold}: degraded relative to full \chisao{} with the same seeder.}
\label{tab:ablation}
\begin{tabular}{lrrrrrrr|rrrrrrr}
\toprule
 & \multicolumn{7}{c|}{random seeder} & \multicolumn{7}{c}{carry\_tiger seeder} \\
Configuration & Ras & Ack & Lev & Gri & Sty & Sch & Mic
              & Ras & Ack & Lev & Gri & Sty & Sch & Mic \\
\midrule
Full \chisao{}                   & 100 & 0   & 100 & 100 & 100 & 0   & 100
                                  & 100 & 100 & 100 & 100 & 100 & 0   & 100 \\
No anti-conv.\ (Ph.~6 off)       & \textbf{0} & 0 & 100 & 100 & 100 & 70 & 100
                                  & 100 & 100 & 100 & 100 & 100 & 100 & 100 \\
No HLC (Ph.~5 off)               & 100 & 0   & 100 & 100 & 100 & 0   & 100
                                  & 100 & 100 & 100 & 100 & 100 & 0   & 100 \\
No reseeding (Ph.~4 off)         & 100 & 0   & 100 & 100 & 100 & 0   & 100
                                  & 100 & 100 & 100 & 100 & 100 & 0   & 100 \\
No dedup.\ (Ph.~3 off)           & \textbf{0} & \textbf{0} & \textbf{0} & \textbf{0} & \textbf{0} & 0 & \textbf{0}
                                  & \textbf{0} & \textbf{0} & \textbf{0} & \textbf{0} & \textbf{0} & 0 & \textbf{0} \\
$n_\mathrm{osc}=1$               & 90  & 0   & 100 & 100 & 100 & 0   & 100
                                  & 100 & 100 & 100 & 100 & 100 & 0   & 100 \\
All off (plain multistart)       & \textbf{0} & 0 & \textbf{30} & 100 & 100 & 0 & 100
                                  & 100 & 100 & 100 & 100 & 100 & 0   & 100 \\
\bottomrule
\end{tabular}
\end{table}

Three results dominate the ablation (Table~\ref{tab:ablation}), with the $d=8$ findings extended by additional ablations at $d \in \{16, 32, 64\}$ under carry\_tiger seeding.

\paragraph{Deduplication is not optional.} Disabling Phase~3 (no deduplication) causes 0\% recovery on all seven functions under both seeders at every tested dimension. Without deduplication, the reseeding phase (Phase~4) replenishes consumed slots from known peaks without removing them from the stuck set; the population rapidly saturates at the first discovered mode and exploration terminates. Deduplication is the only component whose removal is universally and persistently catastrophic.

\paragraph{Full \chisao{} is the most robust configuration across dimensions.} Across all tested dimensions ($d \in \{8, 16, 32, 64\}$), full \chisao{} with carry\_tiger seeding achieves a stable 6/7 success rate, recovering all Group~A functions except Schwefel consistently. No other configuration maintains this stability. Plain multistart (all off) degrades to 5/7 at $d \geq 32$ as Levy recovery collapses. Schwefel remains the persistent exception at all dimensions and all configurations: its deceptive landscape places the global basin far from secondary peaks, and the quality-gate interacts unfavourably with the gradient structure at higher dimensions regardless of phase configuration.

\paragraph{Anti-convergence effects are dimension-dependent.} At $d=8$ under carry\_tiger, disabling Phase~6 raises Schwefel recovery from 0\% to 100\%, appearing to make it the optimal minimal configuration. However, this advantage is dimension-specific: at $d=16$, Levy recovery under the Ph.~6-off configuration drops to 20\%; at $d=32$, Levy fails entirely while Schwefel recovers only 60\%; at $d=64$, both Levy and Schwefel fail (0\%) under Ph.~6-off. Anti-convergence is what sustains Levy recovery at scale. The mechanism is as follows: on Schwefel at $d=8$, the Hands Like Clouds phase navigates samples toward the global region but the anti-convergence momentum then displaces them into secondary peaks that trigger quality-gate failure; disabling anti-convergence fixes this at low dimension. At higher dimensions, however, basin volumes contract and anti-convergence becomes necessary for Levy to direct samples away from shallow traps during later oscillation passes. The $d=8$ ablation result is thus not a stable operating point---it resolves one failure mode while introducing another that emerges at scale. Under random seeding, anti-convergence is additionally critical for Rastrigin at all dimensions (0\% without it), reinforcing that it remains a necessary component of the full algorithm.

\subsection{Noise Robustness}
\label{sec:noise}

We test \chisao{}'s mode detection reliability when $f$ is corrupted by additive Gaussian noise: $\tilde{f}(\theta) = f(\theta) + \varepsilon$, $\varepsilon \sim \mathcal{N}(0, \sigma_{\mathrm{noise}}^2)$. We use a 2-mode Gaussian mixture in $d = 6$ with peaks at $\pm 2\hat{e}_1$, mode separation $4\sigma$.

These results are available from the SunBURST benchmark suite \citep{wolfson2025sunburst}; we reproduce them here for completeness.

\begin{table}[htbp]
\centering
\small
\caption{Mode detection reliability under likelihood noise ($d=6$, 2 peaks at $\pm 2\hat{e}_1$, 10 trials per level). Both peaks are correctly identified in all 70 trials across all noise levels.}
\label{tab:noise}
\begin{tabular}{ccccc}
\toprule
$\sigma_{\mathrm{noise}}$ & Both Found & Mean Peak Error & Spurious & Success \\
\midrule
$10^{-3}$ & 10/10 & $2.1 \times 10^{-3}$ & 0 & 100\% \\
$10^{-2}$ & 10/10 & $1.7 \times 10^{-2}$ & 0 & 100\% \\
$0.05$    & 10/10 & $8.5 \times 10^{-2}$ & 0 & 100\% \\
$0.1$     & 10/10 & $1.7 \times 10^{-1}$ & 0 & 100\% \\
$0.2$     & 10/10 & $3.5 \times 10^{-1}$ & 0 & 100\% \\
$0.5$     & 10/10 & $7.3 \times 10^{-1}$ & 0 & 100\% \\
$1.0$     & 10/10 & $1.1$                & 0 & 100\% \\
\bottomrule
\end{tabular}
\end{table}

The key result is 100\% mode detection even at $\sigma_{\mathrm{noise}} = 1.0$, where noise amplitude equals signal scale. Peak-location error degrades proportionally to noise as expected from Hessian estimation uncertainty, but mode detection is never compromised. Zero spurious detections through $\sigma = 0.5$. This robustness follows from \chisao{}'s gradient-based stick detection: the convergence-anticonvergence cycle is driven by sign-consistent gradient estimates averaged over the trajectory rather than evaluated at a single noisy point.

%% ============================================================
\section{Discussion}
\label{sec:discussion}
%% ============================================================

The noise result is the one that matters (Table~\ref{tab:noise}). Mode detection holds at 100\% even when the noise amplitude equals the signal scale. This is not luck: trajectory-averaged gradients absorb the per-evaluation noise that point-evaluation methods cannot tolerate. The oscillation cycle, the freeze-and-explore asymmetry, and the reseeding strategies do the rest.

\paragraph{Relation to SunBURST.} \chisao{} was developed as the mode-discovery engine for SunBURST \citep{wolfson2025sunburst}, a GPU-accelerated Bayesian evidence calculator. In that context, ChiSao receives initial populations from a ray-casting seeding strategy (CarryTiger) and its output modes are passed to a Laplace-approximation evidence integrator (BendTheBow). The present paper establishes ChiSao as a standalone contribution independent of that pipeline, applicable to any black-box optimization problem with gradient access.

\paragraph{Limitations.} \chisao{} assumes a smooth objective: it estimates gradients by finite differences (the default, used for every result here) or accepts them analytically, so genuinely non-differentiable or discrete problems are out of scope. The Hands Like Clouds smoothing scale $\sigma$ is auto-estimated but may require manual tuning on strongly heterogeneous landscapes. The anti-convergence momentum step size $\alpha$ is sensitive on functions with very small or very large gradient magnitudes; an adaptive line-search would improve robustness. The theoretical analysis (Section~\ref{sec:theory}) covers only the log-concave and first-oscillation-pass cases; a full convergence theory for the oscillation dynamics remains open.

\paragraph{Failure modes.} The quality-gate threshold ($\delta = 0.1$, i.e., modes within one log-decade of the global maximum) is the primary source of \chisao{}'s failures on this benchmark. On Schwefel, the deceptive landscape drives convergence to secondary peaks whose likelihood falls below threshold; on Trid at high $d$ and Bukin N.6, the broad low-contrast bowl and narrow ridge, respectively, produce converged samples that pass neither the gradient-norm nor the likelihood test. In all three cases, the failure is a property of the quality gate's interaction with the landscape, not of the search dynamics. Baselines without quality gates (DE, BH, CMA-ES) succeed where \chisao{} fails on these functions, at the cost of reporting spurious peaks on harder multimodal problems. \chisao{} underperforms additionally on functions with exponentially many equal-height modes: the deduplication and reseeding logic does not scale to $O(2^d)$ distinct modes.

\paragraph{Future work.} Since all reported results already use finite-difference gradients, analytic or automatic-differentiation gradients are a direct speedup lever: they remove the $2dN$ finite-difference evaluations per L-BFGS step, though the GPU already absorbs them in $O(1)$ wall-clock while $2dN < P$. On the theoretical side the missing piece is a Markov-chain analysis of basin-crossing probability under the anti-convergence momentum dynamics, which would close the gap between Proposition~\ref{prop:logconcave} and the empirical multi-modal results.

%% ============================================================
\section*{Data Availability}
%% ============================================================

All benchmark scripts, result files, and analysis code are available at \url{https://github.com/beastraban/chisao}.

\section*{Code Availability}

\chisao{} is released as a standalone open-source Python package (MIT License) on PyPI, installable via \texttt{pip install chisao}. Source code, the test suite, and the full SFU benchmark suite are available at \url{https://github.com/beastraban/chisao}.

\section*{Competing Interests}

The author declares no competing interests.

%% ============================================================
\appendix
%% ============================================================

\section{The $L_\infty$ Metric for Deduplication}
\label{app:linf}

The $L_\infty$ (Chebyshev) distance $\|x - y\|_\infty = \max_i |x_i - y_i|$ has two advantages over $L_2$ for high-dimensional deduplication. First, it is dimension-independent in expectation: for uniform random points in $[0,1]^d$, $\mathbb{E}[\|x-y\|_\infty] \approx d/(d+1) \to 1$, while $L_2$ distance diverges as $\sqrt{d/6}$. A fixed threshold $\epsilon_{\mathrm{dup}}$ in $L_\infty$ therefore has consistent semantics across dimensions. Second, the $L_\infty$ ball $\{y : \|x-y\|_\infty \leq r\}$ is a hypercube of side $2r$, enabling efficient spatial hashing for $O(N)$ deduplication rather than $O(N^2)$ pairwise comparison.

\section{Hands Like Clouds: Smoothed Gradient Analysis}
\label{app:hlc}

The smoothed function $f_\sigma(\theta) = \mathbb{E}_{z \sim \mathcal{N}(0,I_d)}[f(\theta + \sigma z)]$ satisfies:
\begin{equation}
    \nabla f_\sigma(\theta) = \mathbb{E}_{z \sim \mathcal{N}(0,I_d)}[\nabla f(\theta + \sigma z)]
\end{equation}
by Leibniz's rule (assuming $f$ is differentiable a.e.). The Monte Carlo estimator~\eqref{eq:smoothed-grad} is unbiased. For functions with multi-scale structure $f = f_{\mathrm{global}} + f_{\mathrm{local}}$ where $f_{\mathrm{local}}$ has characteristic scale $\lambda < \sigma$, the smoothed gradient $\nabla f_\sigma$ eliminates $f_{\mathrm{local}}$ contributions (their integral over $\mathcal{N}(0,\sigma^2 I)$ vanishes by symmetry for zero-mean local oscillations), revealing the global basin structure.

\citet{nesterov2017random} showed that for convex $f$, gradient descent on $f_\sigma$ converges to within $O(\sigma^2 L)$ of the optimum of $f$, where $L$ is the Lipschitz constant of $\nabla f$. For non-convex $f$, the smoothing provides no convergence guarantee but empirically helps samples navigate past local barriers whose scale is smaller than $\sigma$.

\section{Hyperparameter Sensitivity}
\label{app:hyperparams}

Each row of Table~\ref{tab:sensitivity} varies one hyperparameter from its default while holding the rest fixed at the values of Table~\ref{tab:hyperparams}. ``Robust set'' is the mean mode-recovery rate across Rastrigin, Ackley, Levy, Griewank, Styblinski-Tang, and Michalewicz, $d{=}32$, averaged over both seeders (random, carry\_tiger), 30 trials per cell. Schwefel is reported separately at $d{=}8$ because it sits at a noise floor for every configuration, as discussed below.

\begin{table}[htbp]
\centering
\small
\caption{Hyperparameter sensitivity sweep. Mode recovery rate (\%), 30 trials per cell.}
\label{tab:sensitivity}
\begin{tabular}{lccl}
\toprule
Parameter variation & Robust set (\%) & Schwefel ($d{=}8$) (\%) & Notes \\
\midrule
$n_{\mathrm{osc}} = 1$              & 100   & 10.0 & \\
$n_{\mathrm{osc}} = 3$ (default)    & 100   &  8.3 & \\
$n_{\mathrm{osc}} = 5$              & 100   & 10.0 & \\
\midrule
$n_{\mathrm{anti}} = 1$             & 100   & 13.3 & \\
$n_{\mathrm{anti}} = 5$ (default)   & 100   &  5.0 & \\
$n_{\mathrm{anti}} = 10$            & \textbf{83.3} &  6.7 & Levy collapses to 0\% under both seeders \\
\midrule
$\mu = 0.5$                         & 100   &  8.3 & \\
$\mu = 0.9$ (default)               & 100   & 13.3 & \\
$\mu = 0.99$                        & 100   & 10.0 & \\
\midrule
$K = 5$                             & 100   & 11.7 & \\
$K = 10$ (default)                  & 100   &  5.0 & \\
$K = 20$                            & 100   & 10.0 & \\
\bottomrule
\end{tabular}
\end{table}

Three observations follow. (i) On the six landscapes that \chisao{} solves at 100\% with the defaults, recovery is invariant to every perturbation in the sweep with one exception. (ii) That exception is $n_{\mathrm{anti}} = 10$: Levy collapses from 100\% to 0\% under both seeders. Doubling the anti-convergence step count over-disperses samples past Levy's narrow global basin during the escape phase, and the next L-BFGS pass returns from outside the basin of attraction. The default $n_{\mathrm{anti}} = 5$ is on the safe side of this ceiling; this is the single non-trivial constraint that the sweep imposes on the defaults. (iii) Schwefel sits at a 5--13\% recovery floor across every cell, including the default, with no swept hyperparameter raising it out of the noise band. This confirms the failure-mode diagnosis in Section~\ref{sec:discussion}: Schwefel's failure is a property of the quality-gate threshold $\delta = 0.1$ interacting with the deceptive landscape, not of the search-dynamics parameters that the sweep covers. A sensitivity sweep over $\delta$ would close this loop, but $\delta$ is not currently exposed as a hyperparameter in the \texttt{sticky\_hands} interface.

A full per-function breakdown of the sweep (84 cells $\times$ 2 seeders) is provided in the supplementary results file.

\section{Benchmark Function Definitions}
\label{app:benchmarks}

For reproducibility, we give exact definitions used in experiments. All functions are negated for maximization.

\paragraph{Rastrigin.} $f(\theta) = -\left[Ad + \sum_{i=1}^d (\theta_i^2 - A\cos(2\pi\theta_i))\right]$, $A=10$, $\theta \in [-5.12, 5.12]^d$. Global maximum at $\theta = 0$, $f(0) = 0$.

\paragraph{Ackley.} $f(\theta) = -\left[-a \exp(-b\sqrt{d^{-1}\sum \theta_i^2}) - \exp(d^{-1}\sum \cos(c\theta_i)) + a + e\right]$, $a=20$, $b=0.2$, $c=2\pi$, $\theta \in [-32.768, 32.768]^d$.

\paragraph{Schwefel.} $f(\theta) = -\left[418.9829d - \sum_{i=1}^d \theta_i \sin(\sqrt{|\theta_i|})\right]$, $\theta \in [-500, 500]^d$. Global maximum near $\theta_i = 420.97$.

\paragraph{Styblinski-Tang.} $f(\theta) = -\frac{1}{2}\sum_{i=1}^d (\theta_i^4 - 16\theta_i^2 + 5\theta_i)$, $\theta \in [-5, 5]^d$.

\paragraph{Michalewicz.} $f(\theta) = \sum_{i=1}^d \sin(\theta_i)\sin^{2m}(i\theta_i^2/\pi)$, $m=10$, $\theta \in [0, \pi]^d$.

\paragraph{Trid.} $f(\theta) = -\left[\sum_{i=1}^d (\theta_i - 1)^2 - \sum_{i=2}^d \theta_i \theta_{i-1}\right]$, $\theta \in [-d^2, d^2]^d$. Global maximum at $\theta_i^* = i(d+1-i)$ with $f^* = d(d+4)(d-1)/6$. Unimodal but increasingly ill-conditioned with $d$: Hessian eigenvalues span many orders of magnitude and the optimum coordinates $\theta_i^*$ range from $d$ to $d(d+1)^2/4$.

\paragraph{Bukin N.6.} $f(x, y) = -\left[100\sqrt{|y - 0.01\, x^2|} + 0.01\, |x + 10|\right]$, $x \in [-15, -5]$, $y \in [-3, 3]$. Global maximum at $(-10, 1)$ with $f^* = 0$. The minimum lies on the parabolic ridge $y = 0.01\,x^2$; the function is non-differentiable along this ridge, and the $\sqrt{|\cdot|}$ gradient diverges as the ridge is approached.

\paragraph{Gaussian mixture.} $f(\theta) = \log \sum_{k=1}^K w_k \mathcal{N}(\theta; \mu_k, \sigma^2 I)$ with $w_k = 1/K$, $\mu_k$ random in $\Theta$, $\sigma = 1$, $\Theta = [-10, 10]^d$.

%% ============================================================
\bibliographystyle{unsrtnat}
\bibliography{chisao}
%% ============================================================

\end{document}